\pdfoutput=1
\documentclass[11pt,a4paper]{article}
\usepackage[margin=1in]{geometry}
\usepackage[T1]{fontenc}
\usepackage{lmodern}
\usepackage{amsmath,amssymb,amsthm,booktabs,graphicx}
\usepackage{array,bm,etoolbox}
\usepackage{algorithm,algorithmic}
\usepackage{xcolor,siunitx,float,tabularx,placeins}
\usepackage[font=small,labelfont=bf]{caption}
\usepackage[numbers,sort&compress]{natbib}
\bibpunct[, ]{[}{]}{,}{n}{,}{,}
\usepackage{xurl}
\usepackage[colorlinks=true,linkcolor=blue,citecolor=blue,urlcolor=blue,hypertexnames=false]{hyperref}

\graphicspath{{figures/}}
\renewcommand{\arraystretch}{1.10}
\newtheorem{proposition}{Proposition}

\newcommand{\doi}[1]{\href{https://doi.org/#1}{\nolinkurl{https://doi.org/#1}}}

\newenvironment{keywords}{%
  \par\medskip\noindent\textbf{Keywords: }\ignorespaces
}{\par\medskip}

\makeatletter
\def\NAT@def@citea{\def\@citea{\NAT@separator}}
\makeatother

\begin{document}
\title{\textbf{Vehicle Drift Emergence: Continuous Evolution from Grip Driving to the Handling Limit via Boundary Exploration Learning Model Predictive Control}}
\author{%
  Sheng Zhao\textsuperscript{a}, Binh-Minh Nguyen\textsuperscript{b}, Hangyu Lu\textsuperscript{a,c}, Zexin Wu\textsuperscript{a}, and Xiaodong Wu\textsuperscript{a}\thanks{Corresponding author. Email: xiaodongwu@sjtu.edu.cn}\\[6pt]
  \parbox{0.92\textwidth}{\centering\small
    \textsuperscript{a}School of Mechanical Engineering, Shanghai Jiao Tong University, Shanghai 200240, China\\
    \textsuperscript{b}Department of Advanced Energy, Graduate School of Frontier Sciences, The University of Tokyo, Kashiwa 277-8561, Japan\\
    \textsuperscript{c}Faculty of Engineering, The University of Hong Kong, Pokfulam Road, Hong Kong 999077, China%
  }%
}
 \date{}
\maketitle

\begin{abstract}
Most existing autonomous drifting controllers treat drifting as a pre-programmed maneuver, relying on explicit sideslip references, drift rewards, or discrete mode switching. Yet when a vehicle simply aims to make faster progress on a low-friction circuit, sustained drifting can arise naturally from ordinary grip driving without ever being explicitly commanded. Although this behavior has occasionally been noted as a byproduct of racing or emergency maneuvers, the closed-loop physical mechanism by which grip driving evolves into sustained drifting has rarely been investigated directly. In this paper, we formally characterize this phenomenon, termed \emph{drift emergence}, and propose boundary-exploration learning model predictive control as a unified framework to study it. Starting solely from conservative grip-driving experience under a standard lap-progress objective, the controller learns from completed laps while gradually expanding its allowable handling range within a single vehicle and tire model. Simulation results show that sustained drifting develops continuously across repeated corner passages and spreads around the circuit as driving speed increases. Tracing the same corner across laps reveals that the transition is governed by the rear tires approaching friction saturation and losing lateral stiffness under rising cornering demand. Comparative studies further confirm that drift emergence is shaped by accumulated driving experience, road friction capacity, and nonlinear tire modeling, while remaining robust to tire-model discrepancies, state measurement noise, and real-time Speedgoat hardware-in-the-loop execution. Overall, these findings demonstrate that drifting need not be viewed as a separate operating mode, but as a continuous physical extension of limit handling when rear-tire capacity is sustained near its limit.
\end{abstract}


\section{Introduction}\label{sec:introduction}

Operating an automated vehicle near the tire--road friction limit is central to both agile driving on low-adhesion surfaces and active-safety evasive maneuvers~\cite{bertipaglia2025,weber2025lowfriction}. In this limit-handling regime, nonlinear tire saturation can push a vehicle beyond conventional grip cornering into large-sideslip sliding, where diminishing lateral force margin may either precipitate loss of control~\cite{mastinu2024losecontrol} or be harnessed as sustained drift~\cite{stano2024,leng2026review,wu2026driftsurvey}. Understanding how sustained drift arises naturally from ordinary grip driving under a standard motion task---without a prescribed sideslip reference, explicit drift reward, or discrete mode switch---is therefore of fundamental significance for vehicle dynamics and limit-handling control. Clarifying this transition reveals whether drifting should be viewed merely as a separately scheduled operating mode or as a continuous physical continuation of rear-axle tire saturation when task demand approaches available tire--road capacity.

Most existing studies of high-sideslip vehicle dynamics and autonomous drifting focus on analyzing unstable drift equilibria and designing controllers to stabilize or track prescribed drifting motions~\cite{voser2010highsideslip,kritayakirana2012,leng2026review,wu2026driftsurvey}. Recent contributions span nonlinear equilibrium stability and powerslide torque distribution~\cite{eberhart2026powerslide}, controlled transitions among high-sideslip maneuvers such as donuts and drifting~\cite{zhang2026donuts}, and drift-equilibrium stabilization for sharp turns~\cite{jia2024drift}. Model-based and preview feedback controllers have further advanced transient drift tracking, active-safety drifting, unified grip-and-drift path following, and driver drift assistance~\cite{goh2024,weber2024,shi2023,stano2024,sun2025driftassist,gan2026preview,song2026unified}, while reinforcement learning methods have demonstrated reference-free or sim-to-real drifting policies trained with drift-oriented rewards~\cite{djeumou2024referencefree,toth2024simtoreal}. Although classical and recent lateral-dynamics analyses link large-sideslip behavior to combined-slip saturation and vanishing incremental cornering stiffness~\cite{pacejka2012,rajamani2012,huangregion2020,sun2026lateral}, the prevailing drift-control literature still treats drifting as a pre-specified control objective rather than investigating how drift emerges organically from grip driving.

Only a small body of work has noted that drift can appear without an explicit drifting target. Tavernini et al.~\cite{tavernini2013minimum} showed that offline minimum-time optimal-control solutions can select high-drift cornering on loose surfaces depending on road friction and transmission layout, and Weiss et al.~\cite{weiss2026emergent} recently observed task-induced emergent drifting during closed-loop winter collision avoidance. However, while these studies note or exhibit drift as a byproduct of time-optimal planning or emergency evasion, they do not conduct a dedicated dynamical investigation of the emergence process itself. Specifically, it remains unclear how a closed-loop vehicle starting purely from grip driving evolves across repeated corner passages into sustained drift, which coupled state and rear-axle tire signatures mark its onset, and how handling bounds, accumulated experience, and tire--road friction capacity govern that transition. To the best of the authors' knowledge, this paper is among the first dedicated studies to formalize and systematically investigate \emph{drift emergence} as a closed-loop vehicle-dynamics phenomenon rather than a prescribed maneuver or incidental outcome.

To probe this grip-to-drift transition in closed loop, we formulate boundary-exploration learning model predictive control (BE-LMPC) as an exploratory tool. Extending iterative learning MPC for racing~\cite{rosolia2018,rosolia2020}, BE-LMPC begins solely from low-sideslip grip experience and couples empirical terminal-set updates with lap-wise expansion of predicted sideslip and yaw-rate bounds under a pure circuit-progress objective. A single combined-slip vehicle and tire model governs both grip and drift regimes with no drift reference, sideslip reward, or mode switch. Following the same corners across laps quantifies when the envelope admits large sideslip, when sustained drift emerges, how rear-axle utilization and incremental stiffness evolve, and how drift spreads under friction shifts, perturbations, and hardware-in-the-loop (HIL) tests.

The study makes two main contributions:
\begin{enumerate}
  \item \textbf{Definition and dynamical characterization of drift emergence:} We formalize \emph{drift emergence}---the continuous closed-loop development of sustained drift from grip driving under a progress objective without an explicit drift target or reward---and establish its combined-slip continuity conditions, offline persistence criteria, and spatially matched signatures in sideslip, yaw rate, rear-axle utilization, and incremental stiffness.
  \item \textbf{BE-LMPC exploratory framework and experimental validation:} We propose BE-LMPC as an exploratory framework to elicit and study grip-to-drift evolution, and validate the mechanism through simulation ablations (handling bounds, frozen experience, racing LMPC prediction, friction capacity, tire-law mismatch, and state noise) and real-time Speedgoat HIL tests.
\end{enumerate}

\section{Vehicle Dynamics Model and Assumptions}\label{sec:vehicle-model}

\subsection{Vehicle Configuration and State Definition}

The study uses a planar single-track model with front steering and combined-slip tires~\cite{goh2019toward}. Figure~\ref{fig:vehicle_model} shows the axle geometry and sign convention. The state and commanded inputs are
\begin{equation}
\bm{x}=[X,Y,\psi,v,r,\beta]^\mathsf T,\qquad
\bm{u}=[a_{\rm cmd},\delta]^\mathsf T .
\end{equation}
Here, $(X,Y)$ is the center-of-gravity position, $\psi$ is heading, $v$ is speed, $r$ is yaw rate, $\beta$ is vehicle sideslip, and $\delta$ is front steering. Positive heading, yaw rate, sideslip and steering are counterclockwise. The distances from the center of gravity to the axles are $l_f$ and $l_r$, and $I_z$ is the yaw moment of inertia.

\begin{figure}[ht]
  \centering
  \includegraphics[width=76mm]{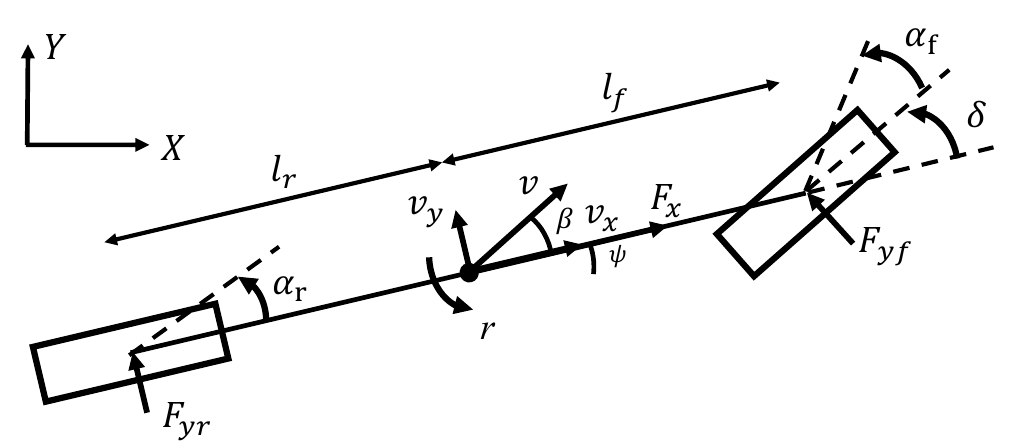}
  \caption{Single-track geometry and vehicle sideslip angle $\beta$.}
  \label{fig:vehicle_model}
\end{figure}

Let $\tilde v=\max(v,v_{\min})$ in the dynamic branch, with $v_{\min}=\SI{0.15}{m/s}$. The lumped propulsion model determines the speed rate $a_v$. The body-lateral projection of the front lateral force and the rear lateral force determine yaw and sideslip:
\begin{equation}
\begin{aligned}
\dot X &= \tilde v\cos(\psi+\beta),&
\dot Y &= \tilde v\sin(\psi+\beta),\\
\dot\psi &= r,&
\dot v &= a_v,\\
\dot r &= (l_f F_{yf}\cos\delta-l_r F_{yr})/I_z,\\
\dot\beta &= (F_{yf}\cos\delta+F_{yr})/(m\tilde v)-r.
\end{aligned}
\label{eq:vehicle}
\end{equation}
Longitudinal axle forces couple into the lateral dynamics through the combined-slip tire limit, and \eqref{eq:vehicle} adopts a reduced-order lateral force projection. The tire slip angles are
\begin{equation}
\begin{split}
\alpha_f &= \operatorname{atan2}(v_y+l_fr,v_x)-\delta,\\
\alpha_r &= \operatorname{atan2}(v_y-l_rr,v_x),\\
v_x &= \max(\tilde v\cos\beta,v_{\min}),\qquad
v_y=\tilde v\sin\beta .
\end{split}
\label{eq:slip_angles}
\end{equation}
A cubic smoothstep blends the dynamic and kinematic vector fields between \SI{0.45}{m/s} and \SI{0.90}{m/s}. The low-speed branch propagates position in the heading direction, uses $r_{\rm kin}=v\tan\delta/(l_f+l_r)$ for the heading rate, and relaxes $r$ toward $r_{\rm kin}$ and $\beta$ toward zero with time constants 0.12 and 0.10~s. The dynamic branch has unit weight at and above \SI{0.90}{m/s}.

\subsection{Longitudinal Actuation and Axle Loads}

The requested force is $F_x^{\rm req}=ma_{\rm cmd}$. The lumped drive and brake actuator applies
\begin{equation}
F_x=
\begin{cases}
\min(F_x^{\rm req},F_{d,\max},F_{p,\max},\mu mg), & a_{\rm cmd}\geq0,\\
\max(F_x^{\rm req},-F_{b,\max},-\mu mg), & a_{\rm cmd}<0,
\end{cases}
\label{eq:longitudinal_force}
\end{equation}
where $\mu$ is the tire--road friction coefficient and $F_{d,\max}$ and $F_{b,\max}$ are actuator limits. With available wheel power $P_{\max}$ and regularizing speed $v_\epsilon$, the speed-dependent force bound and net acceleration are
\begin{equation}
\begin{split}
F_{p,\max}&=\frac{P_{\max}}{\max(v,v_\epsilon)},\\
a_v&=\frac{F_x-F_{\rm res}}{m},\qquad
F_{\rm res}=c_{\rm rr}mg\,\sigma(v)+c_dv^2.
\end{split}
\label{eq:power_limit}
\end{equation}
The rolling resistance term uses a cubic smoothstep $\sigma$ from zero to one between zero and \SI{0.20}{m/s}, with the force map evaluated at the regularized speed where applicable. Static weight distribution and longitudinal load transfer define the axle normal loads as
\begin{equation}
\begin{split}
F_{zf}&=\max\!\left(0.05mg,\frac{m(gl_r-a_vh)}{l_f+l_r}\right),\\
F_{zr}&=\max\!\left(0.05mg,\frac{m(gl_f+a_vh)}{l_f+l_r}\right).
\end{split}
\label{eq:normal_loads}
\end{equation}
Here, $h$ is the center-of-gravity height. The axle forces are allocated in proportion to these loads, $F_{xi}=F_x F_{zi}/(F_{zf}+F_{zr})$ for $i\in\{f,r\}$. The combined-force utilization is
\begin{equation}
\rho_i=\frac{\sqrt{F_{xi}^2+F_{yi}^2}}{\mu F_{zi}}.
\label{eq:utilization}
\end{equation}

\subsection{Combined-Slip Tire Model}

For $F_{zi}>0$ and $|F_{xi}|\leq\mu F_{zi}$, longitudinal loading leaves the residual lateral capacity
\begin{equation}
\bar\mu_iF_{zi}=
\sqrt{(\mu F_{zi})^2-F_{xi}^2},
\qquad i\in\{f,r\}.
\label{eq:lateral_capacity}
\end{equation}
This coupling between longitudinal demand and available cornering force is central to tire-force control near the friction limit~\cite{Fuse2023cornering}. For positive residual capacity, the Fiala lateral force is
\begin{equation}
F_{yi}=
\begin{cases}
-\bar\mu_iF_{zi}
(3h_i-3|h_i|h_i+h_i^3),& |h_i|<1,\\
-\bar\mu_iF_{zi}\operatorname{sgn}(\alpha_i),& |h_i|\geq1,
\end{cases}
\label{eq:fiala}
\end{equation}
where
\begin{equation}
h_i=\frac{C_{\alpha i}\tan\alpha_i}
{3\bar\mu_iF_{zi}},\qquad C_{\alpha i}=c_{\alpha i}F_{zi}.
\label{eq:fiala_h}
\end{equation}
The constants $c_{\alpha i}$ specify cornering stiffness per unit normal load. When $|F_{xi}|=\mu F_{zi}$, the residual lateral capacity vanishes and $F_{yi}=0$; together with the longitudinal force bound and load-proportional allocation above, this guarantees $F_{xi}^2+F_{yi}^2\leq(\mu F_{zi})^2$.

Within each run, the road surface is spatially uniform and time-invariant. Unless varied in the friction-mismatch or tire-law tests of Section~\ref{sec:perturbations}, the Fiala law serves for BE-LMPC prediction, online force estimation, and the simulated plant, capturing longitudinal load transfer and axle-level combined-slip limits under planar motion with lumped actuation.

\section{Drift Emergence and Its Physical Signature}\label{sec:formulation}

\subsection{Drift Emergence and Offline Event Identification}

Drift emergence denotes the development of sustained large vehicle sideslip with yaw motion and high rear-axle force demand during ordinary task execution. The initial experience comes from grip driving. The plant and control law are the same across the transition, and the objective contains neither a prescribed drift state nor a reward for drifting. Spatially matched trajectories reveal how the response develops at the same corners across repeated laps.

For reproducible event reporting, an offline indicator combines sideslip, yaw motion, and rear-axle utilization in the candidate set
\begin{equation}
\begin{split}
\mathcal D=\{(\bm{x},\bm{u}):&
|\beta|\geq\beta_{\rm th},\
|r|\geq r_{\rm th},\\
&\rho_r\geq\rho_{\rm th}\}.
\end{split}
\label{eq:drift_set}
\end{equation}
A sustained drift episode is registered when the trajectory remains in $\mathcal D$ for at least $T_{\rm d}=\SI{0.30}{s}$ (six consecutive samples at the \SI{0.05}{s} control period), with nominal thresholds $\beta_{\rm th}=8^\circ$, $r_{\rm th}=\SI{0.45}{rad/s}$, and $\rho_{\rm th}=0.82$ (where $\rho_r=1$ denotes full rear-axle saturation). For each lap, $D_{\max}$ denotes the duration of the longest consecutive qualifying interval. Computed strictly offline from recorded plant trajectories, this indicator affects neither the online objective nor lap acceptance; its sensitivity to threshold choices is examined in Section~\ref{sec:results}.

\subsection{Combined-Slip Response Near the Handling Limit}

To establish continuity for a general tire-force map, we write the planar dynamics from Section~\ref{sec:vehicle-model} compactly as
\[
\dot{\bm x}=\bm f(\bm x,\bm u,\bm\theta),
\]
where $\bm f$ denotes the vector field of the planar vehicle model. The vector $\bm\theta$ collects the external model parameters, including the tire-road friction coefficients.
\begin{proposition}[Sufficient regularity for continuous grip-to-drift evolution]
Let the axle lateral forces be $F_{yi}=\Phi_i(\alpha_i,F_{zi},F_{xi},\mu_i)$. Consider the forward-motion domain specified in Section~\ref{sec:vehicle-model}, with $v\geq v_{\min}>0$, $F_{zi}>0$, and $|F_{xi}|<\mu_iF_{zi}$. Assume that longitudinal force and normal load are locally Lipschitz functions of the state, input and model parameters. Suppose that each $\Phi_i$ is locally Lipschitz in its arguments, the control input is bounded and piecewise continuous, and the external parameters, including $\mu_i$, vary continuously in time. Then the vector field $\bm f$ is locally Lipschitz in the state for each fixed time. In particular, crossing an adhesion-to-sliding branch boundary of a continuous tire-force map does not itself introduce a discontinuity in $\bm f$, and the associated state trajectory remains continuous.
\end{proposition}

\begin{proof}
On the stated domain, the slip-angle maps and the kinematic terms in \eqref{eq:vehicle} are locally Lipschitz. The actuator and load maps use locally Lipschitz bounds, and the low-speed blend is smooth. Composition with the tire map gives a vector field that is locally Lipschitz in the state and piecewise continuous in time. The associated Carath\'eodory solution is absolutely continuous.
\end{proof}

The combined-slip Fiala map in \eqref{eq:fiala} provides an analytical verification of the tire regularity condition. When $F_{zi}>0$ and $|F_{xi}|<\mu F_{zi}$, the residual lateral capacity satisfies $\bar\mu_iF_{zi}>0$. The force and its first derivative with respect to $\alpha_i$ then match at $|h_i|=1$. Holding normal load, longitudinal force, and friction fixed, the incremental cornering stiffness on the unsaturated branch is

\begin{equation}
C_{i,t}=-\frac{\partial F_{yi}}{\partial\alpha_i}
=C_{\alpha i}\sec^2\alpha_i(1-|h_i|)^2
\label{eq:tangent}
\end{equation}
and tends continuously to zero at $|h_i|=1$, matching the saturated branch. The force also matches there, so the Fiala map is $C^1$ in slip angle at sliding onset.

Figure~\ref{fig:mechanism}(a) shows how longitudinal loading changes the Fiala stiffness map. As $\chi_x=|F_x|/(\mu F_z)$ rises from 0 to 0.8, sliding begins at a smaller tire slip angle and the normalized lateral-force plateau falls from 1.00 to 0.60. Near the plateau, the smaller incremental stiffness means that a given increase in lateral force requires a larger change in tire slip.

Figure~\ref{fig:mechanism}(b) follows BE-LMPC at $s=\SI{37.8}{m}$, where absolute sideslip peaks within the Lap~11 drift episode at this corner. From Lap~9 to Lap~12, vehicle sideslip grows from $11.4^\circ$ to $16.4^\circ$ while yaw rate stays near $-\SI{1.73}{rad/s}$. Rear utilization first reaches 0.82 at this station on Lap~10. The joint conditions first persist long enough to qualify as sustained drift on Lap~11 (Section~\ref{sec:results}).

\begin{figure}[!t]
  \centering
  \includegraphics[width=142mm]{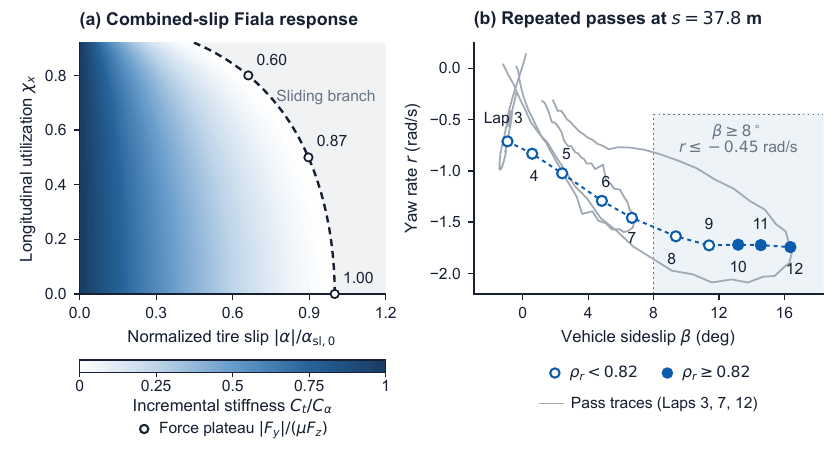}
  \caption{Combined-slip tire response and cross-lap evolution at $\mu=0.60$. (a) Incremental stiffness for $C_\alpha/F_z=2.3\,{\rm rad}^{-1}$, with tire slip scaled by the pure-lateral sliding angle $\alpha_{\rm sl,0}=\arctan(3\mu F_z/C_\alpha)$. The dashed curve marks sliding onset, and the labels give the normalized lateral-force plateau. (b) Numbered points show the states interpolated at $s=\SI{37.8}{m}$ on Laps~3--12. Filled points meet the rear-utilization threshold, and the shaded region indicates where both the sideslip and yaw-rate thresholds are met for this turn. Gray traces show recorded samples within \SI{5}{m} of the station on Laps~3, 7, and~12.}
  \label{fig:mechanism}
\end{figure}

\subsection{Conditionality of Drift Emergence}

Whether the continuous response develops into drift depends on the force demanded by the motion and the capacity available at the tires. For a path with local curvature $\kappa$, steady grip operation in the absence of longitudinal force has the normalized lateral demand
\begin{equation}
\Lambda_y=\frac{v^2|\kappa|}{\mu g},
\label{eq:demand}
\end{equation}
The associated ideal grip-speed boundary is
\begin{equation}
v_{\rm grip}=\sqrt{\frac{\mu g}{|\kappa|}}.
\label{eq:grip_speed}
\end{equation}
Longitudinal force further reduces the available lateral capacity, while propulsion and braking limits in \eqref{eq:longitudinal_force}, track geometry, and the control envelope bound the attainable speed $v_{\rm att}$. Defining the demand-to-capacity ratio
\begin{equation}
\mathcal R=\frac{v_{\rm att}}{v_{\rm grip}},
\end{equation}
a regime with $\mathcal R\ll1$ remains well below the grip-speed boundary. Lower road friction reduces $v_{\rm grip}$, whereas iterative learning raises $v_{\rm att}$ and increases cornering demand alongside braking, path selection, and transient load transfer. This ratio motivates the friction tests below, while the measured rear utilization \eqref{eq:utilization} directly quantifies realized tire loading.

\section{Boundary-Exploration Learning Model Predictive Control}\label{sec:controller}

Figure~\ref{fig:framework} links the circuit-progress task and BE-LMPC's lap-wise updates to the emergence of sustained drift. Section~\ref{sec:results} uses the resulting trajectories to quantify drift onset at a matched corner and its spread across the circuit.

\begin{figure*}[ht]
  \centering
  \includegraphics[width=\textwidth]{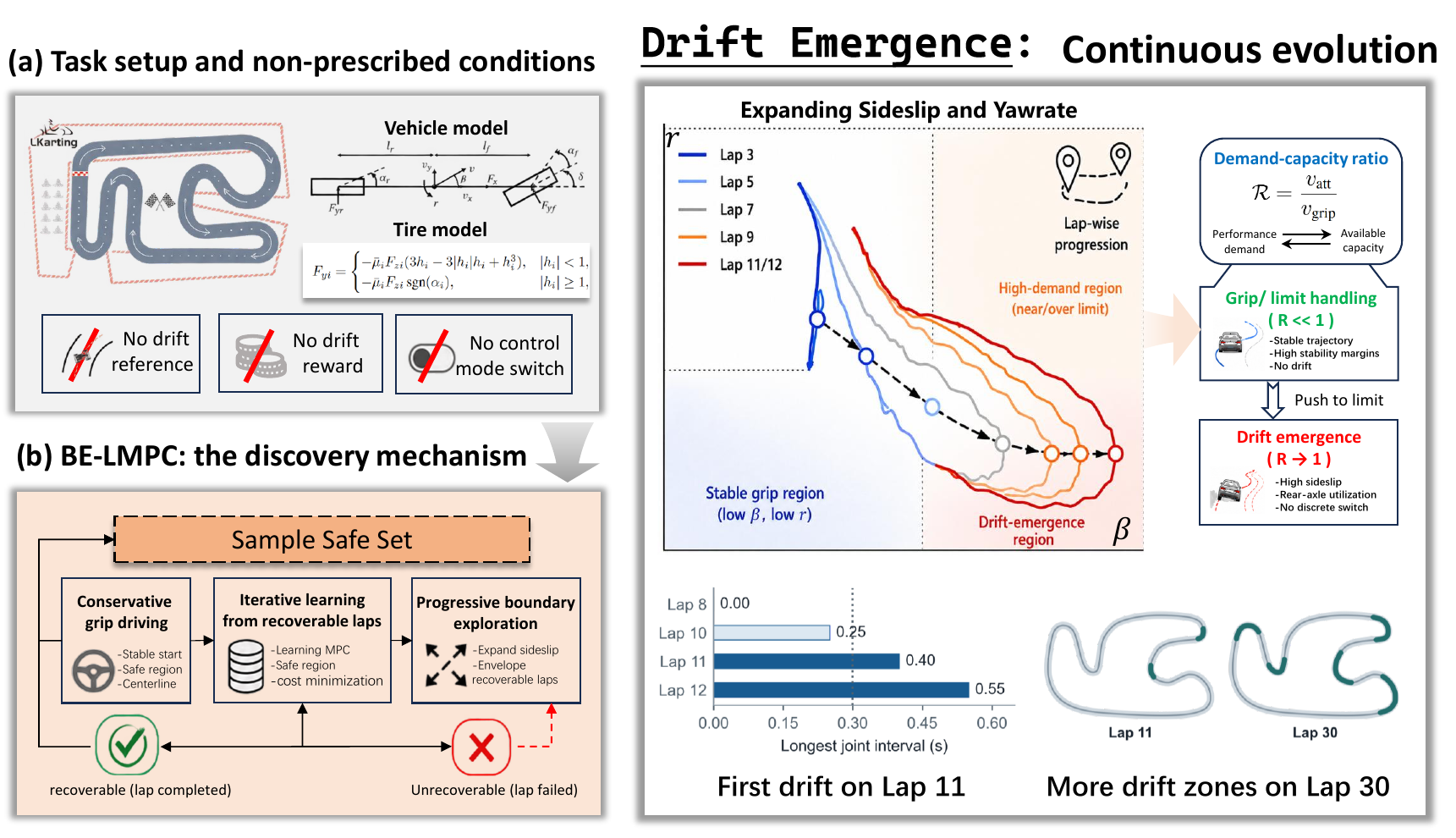}
  \caption{Framework for studying drift emergence. (a) Circuit progress from grip-driving experience without a drift objective. (b) Lap-wise updates of terminal experience and handling bounds. The schematic phase plane and demand-to-capacity diagram illustrate the approach to the handling limit. The nominal simulation at $\mu=0.60$ shows onset at corner A and spatial spread from Lap~11 to Lap~30.}
  \label{fig:framework}
\end{figure*}

\subsection{Repetitive Task and Empirical Safe Set}

Consider a vehicle that repeatedly performs a lap-time-improvement task on a closed track. A completed trajectory with $N_j$ samples on Lap~$j$ is
\begin{equation}
\mathcal T^j=
\{\bm{x}_k^j,\bm{u}_k^j,s_k^j\}_{k=0}^{N_j-1},
\end{equation}
where $s$ denotes track arc length. Each sample has the empirical remaining-step cost
\begin{equation}
Q_k^j=N_j-1-k .
\label{eq:cost_to_go}
\end{equation}
Completed LMPC iterations provide empirical terminal experience~\cite{rosolia2018,rosolia2020}. Each state retained from an accepted lap has an observed continuation to lap completion.

The present implementation uses a local approximation instead of all recorded states. At each control instant, it selects 16 neighboring samples from each of the two most recent accepted laps. With index set $\mathcal I_t$, their convex combinations define the local empirical safe set
\begin{equation}
\mathcal{CS}_t=
\left\{
\sum_{i\in\mathcal I_t}\lambda_i\bm{x}_i:
\lambda_i\geq0,\ 
\sum_{i\in\mathcal I_t}\lambda_i=1
\right\}.
\label{eq:safe_set}
\end{equation}
This local construction keeps the terminal optimization compact. The whole-lap test below defines which completed trajectories supply terminal samples.

To improve numerical conditioning, the implementation shifts the costs separately within the selected neighborhood of each source lap. Specifically, $\widetilde Q_i=Q_i-\min_{p\in\mathcal I_t^{j(i)}}Q_p\geq0$, where $j(i)$ identifies the lap containing sample $i$. The terminal progress cost is
\begin{equation}
Q_{\rm term}(\bm\lambda)=
\sum_{i\in\mathcal I_t}\lambda_i\widetilde Q_i\geq0.
\label{eq:terminal_cost}
\end{equation}
A high-weight soft constraint places the predicted terminal state near $\mathcal{CS}_t$. The local cost favors progress within the selected track neighborhood.

\subsection{Boundary Exploration and Whole-Lap Acceptance}

BE-LMPC adjusts the admissible handling envelope from lap to lap:
\begin{equation}
\mathcal E(\eta_j)=
\left\{
\bm{x}:\begin{aligned}
&0\leq v\leq v_{\max},\quad
|\beta|\leq\bar\beta(\eta_j),\\
&|r|\leq\bar r(\eta_j)
\end{aligned}
\right\},
\end{equation}
where
\begin{equation}
\bar z(\eta_j)=
\bar z_0+\eta_j(\bar z_{\max}-\bar z_0),
\quad z\in\{\beta,r\}.
\end{equation}
While $v_{\max}$ remains fixed, expanding $\eta_j\in[0,1]$ widens the admissible sideslip and yaw-rate bounds for online optimization. After each completed lap, $\eta_j$ is updated as
\begin{equation}
\eta_{j+1}=
\operatorname{clip}_{[0,1]}
\begin{cases}
\eta_j+\Delta_+, & \mathcal T^j\text{ is recoverable},\\
\eta_j-\Delta_-, & \text{otherwise}.
\end{cases}
\label{eq:envelope}
\end{equation}
A completed lap is accepted if its minimum signed track margin is at least $-\SI{0.02}{m}$, its solver-failure fraction does not exceed 0.15, its peak combined utilization stays within 1.08, and all recorded states remain finite. Only accepted laps enter the empirical safe set, accommodating discretization and linearization errors while excluding failed exploratory trajectories.

\subsection{Online Optimization}

At control instant $t$, the algorithm linearizes the nonlinear model once along the warm-start trajectory retained from the preceding MPC solution. Zero-order-hold discretization gives
\begin{equation}
\bm{x}_{k+1|t}=\bm A_{k|t}\bm{x}_{k|t}
+\bm B_{k|t}\bm{u}_{k|t}+\bm c_{k|t}.
\label{eq:affine_prediction}
\end{equation}
The resulting linear time-varying quadratic program (QP) optimizes the predicted states $\bm{x}_{0:N|t}$ and inputs $\bm{u}_{0:N-1|t}$. Its remaining decision variables are the scalar track slacks $\epsilon_{0:N}$, convex weights $\bm\lambda$, and terminal slack $\bm e_N\in\mathbb R^6$. The objective is
\begin{equation}
\begin{split}
\min\quad
&\frac{1}{2}\sum_{k=0}^{N-1}
\bm u_{k|t}^{\mathsf T}\bm R\bm u_{k|t}
+Q_{\rm term}(\bm\lambda)\\
&+\frac{q_s}{2}\sum_{k=1}^{N}\epsilon_{k|t}^2
+\frac{q_N}{2}\|\bm e_N\|_2^2,
\end{split}
\label{eq:objective}
\end{equation}
where $\bm R=\operatorname{diag}(1.5,18)$, $q_s=3000$, and $q_N=800$. The penalty $q_N$ softens terminal-state matching. The same objective weights apply to all experiments. The QP is subject to
\begin{subequations}
\label{eq:ocp_constraints}
\begin{align}
\bm x_{0|t}&=\bm x_t,\qquad
\bm x_{k|t}\in\mathcal E(\eta_j),
\label{eq:ocp_initial}\\
\bm x_{k+1|t}&=\bm A_{k|t}\bm x_{k|t}
+\bm B_{k|t}\bm u_{k|t}\notag\\
&\quad+\bm c_{k|t},
\label{eq:ocp_dynamics}\\
-a_{b,\max}\leq a_{{\rm cmd},k|t}&\leq a_{d,\max},\notag\\
|\delta_{k|t}|&\leq\delta_{\max},
\label{eq:ocp_input}\\
|\delta_{0|t}-\delta_{t-1}|&\leq\dot\delta_{\max}T_s,\notag\\
|\delta_{k|t}-\delta_{k-1|t}|&\leq\dot\delta_{\max}T_s,
\label{eq:steering_rate}\\
\underline b_{k|t}-\epsilon_{k|t}
&\leq\bm n_{k|t}^{\mathsf T}\bm p_{k|t}
\leq\overline b_{k|t}+\epsilon_{k|t},
\label{eq:ocp_track}\\
-\bm e_N&\leq\bm x_{N|t}
-\sum_{i\in\mathcal I_t}\lambda_i\bm x_i\notag\\
&\leq\bm e_N,
\label{eq:ocp_terminal}\\
\epsilon_{k|t}&\geq0,\quad \bm e_N\geq\bm0,\quad
\lambda_i\geq0,\notag\\
&\sum_{i\in\mathcal I_t}\lambda_i=1.
\label{eq:ocp_nonnegative}
\end{align}
\end{subequations}
Here, $\bm p=[X,Y]^\mathsf T$. The vector $\bm n_{k|t}$ and bounds $\underline b_{k|t},\overline b_{k|t}$ define the locally linearized track corridor. Force saturation and resistance enter through the nonlinear model used to construct \eqref{eq:affine_prediction}. OSQP solves the QP~\cite{stellato2020}, and the first optimized input is applied to the nonlinear plant.

In \eqref{eq:objective}, the terminal cost rewards lap progress while the remaining terms penalize control effort, corridor slack, and deviation from the convex hull of recorded terminal states. High-sideslip states enter this terminal memory only after the controller has already generated them on an accepted lap, with no explicit drift reference, sideslip reward, or mode switch in the cost. If the QP fails to solve, a standard track-following and speed-recovery controller takes over.

\begin{algorithm}[t]
\caption{Closed-loop exploration for the drift-emergence test}
\label{alg:be_lmpc}
\begin{algorithmic}[1]
\STATE \textbf{Input:} initial completed-lap set $\mathcal A^0$, initial boundary level $\eta_0$, and steps $\Delta_+,\Delta_-$
\FOR{learning lap $j=1,2,\ldots,J$}
  \FOR{within-lap control instant $k=0,1,\ldots$}
    \STATE Construct $\mathcal{CS}_t$ and $Q_{\rm term}$ from the two most recent accepted laps
    \STATE Solve \eqref{eq:objective} within $\mathcal E(\eta_j)$
    \STATE Apply the first input and record $\bm{x}_k^j,\bm{u}_k^j$
  \ENDFOR
  \STATE Check the track margin and solver status; verify friction utilization and state finiteness
  \IF{the complete lap is recoverable}
    \STATE $\mathcal A^j\leftarrow\mathcal A^{j-1}\cup\{\mathcal T^j\}$
    \STATE $\eta_{j+1}\leftarrow\min(1,\eta_j+\Delta_+)$
  \ELSE
    \STATE $\mathcal A^j\leftarrow\mathcal A^{j-1}$; reject the failed trajectory
    \STATE $\eta_{j+1}\leftarrow\max(0,\eta_j-\Delta_-)$
  \ENDIF
  \STATE Evaluate the event indicator offline for reporting only
\ENDFOR
\end{algorithmic}
\end{algorithm}

Algorithm~\ref{alg:be_lmpc} summarizes the closed-loop procedure.

\section{Simulation Validation}\label{sec:simulation}

\subsection{Implementation and Experimental Protocol}

\subsubsection{Simulation Setup and Drift Measures}

The circuit is \SI{80.48}{m} long with a \SI{0.75}{m} half-width, and two conservative laps provide the initial experience. Unless otherwise noted, each test spans ten learning laps (Laps~3--12) using a \SI{0.05}{s} control period and fourth-order Runge--Kutta plant integration with \SI{0.005}{s} substeps. Table~\ref{tab:parameters} lists the vehicle and controller parameters; the driving objective, cost structure, and weights remain fixed across all comparisons.

\begin{table}[!t]
\centering
\small
\renewcommand{\arraystretch}{1.15}
\setlength{\tabcolsep}{4pt}
\caption{Principal vehicle and controller parameters}
\label{tab:parameters}
\begin{tabularx}{\textwidth}{
  @{}
  >{\raggedright\arraybackslash}X c l
  @{\hspace{1.5em}}
  >{\raggedright\arraybackslash}X c l
  @{}
}
\toprule
\multicolumn{3}{c}{Vehicle parameters}
&
\multicolumn{3}{c}{Controller parameters}
\\
\cmidrule(lr){1-3}\cmidrule(lr){4-6}
Parameter & Symbol & Value
&
Parameter & Symbol & Value
\\
\midrule

Vehicle mass
& $m$
& \SI{3.17}{kg}
&
Power regularization
& $v_\epsilon$
& \SI{0.50}{m/s}
\\

Yaw moment of inertia
& $I_z$
& \SI{0.03984}{kg.m^2}
&
Prediction horizon
& $N$
& 25
\\

CG-to-axle distances
& $l_f/l_r$
& 0.15875/0.17145 m
&
Control period
& $T_s$
& \SI{0.05}{s}
\\

CG height
& $h$
& \SI{0.08255}{m}
&
Global speed bound
& $v_{\max}$
& \SI{10.0}{m/s}
\\

Cornering stiffness/load
& $c_{\alpha f}/c_{\alpha r}$
& 2.3/2.3 rad$^{-1}$
&
Maximum steering angle
& $\delta_{\max}$
& \SI{0.41}{rad}
\\

Drive/brake force limit
& $F_{d,\max}/F_{b,\max}$
& 12.7/15.9 N
&
Maximum steering rate
& $\dot{\delta}_{\max}$
& \SI{1.0}{rad/s}
\\

Maximum wheel power
& $P_{\max}$
& \SI{24}{W}
&
Sideslip (init./max.)
& $\bar{\beta}_0/\bar{\beta}_{\max}$
& $7^\circ/38^\circ$
\\

Roll-resistance coefficient
& $c_{\rm rr}$
& 0.018
&
Yaw-rate (init./max.)
& $\bar r_0/\bar r_{\max}$
& 2.4/6.0 rad/s
\\

Quadratic drag coeff.
& $c_d$
& \SI{0.045}{N.s^2/m^2}
&
Envelope steps (+/-)
& $\Delta_+/\Delta_-$
& 0.20/0.10
\\

\bottomrule
\end{tabularx}
\end{table}

Sustained drift is identified offline from recorded sideslip, yaw rate, and rear-axle utilization using the criteria in Section~\ref{sec:formulation}, reporting the first qualifying lap and the longest joint interval $D_{\max}$ per lap. Comparing passages at the same track station relates the growth in sideslip to local speed, rear utilization, and normalized incremental rear stiffness, while track margins, solver status, and lap time quantify completion quality and task progress.

To examine longer-term evolution, the nominal $\mu=0.60$ run is extended from Lap~12 to Lap~30. We map each qualifying episode onto the circuit, report both total and longest episode duration per lap, and additionally evaluate the rear slip index $|h_r|$ and the ratio $|F_{yr}|/(\bar\mu_r F_{zr})$ at $s=\SI{37.8}{m}$.

\subsubsection{Matched Controller Comparisons}

Two variants restrict the admissible handling range. The grip-constrained variant keeps the predicted sideslip bound at $7^\circ$ while allowing yaw-rate expansion. The fixed narrow variant retains both initial bounds. These variants and their BE-LMPC comparison start from the same two \SI{1.2}{m/s} path-following records and use the same lap acceptance and fallback rules.

The fixed-wide variant uses the full bounds of $38^\circ$ and \SI{6.0}{rad/s} from the first learning lap. The frozen-experience variant follows BE-LMPC through the first completed lap at the maximum envelope and retains the same two terminal source trajectories thereafter. In the paired BE-LMPC run, accepted laps continue to update the terminal set under the same full bounds. This comparison isolates the role of new terminal experience once the envelope is fully open.

The fixed-wide and frozen-experience comparisons, the racing LMPC comparison, and the fixed-estimate and perturbation tests use two full-state conservative laps generated with the nominal Fiala plant at $\mu=0.60$. The speed commands are centered at \SI{1.15}{m/s} and \SI{1.30}{m/s} with a small spatial modulation. Both seed laps stay below $8^\circ$ of absolute sideslip. Actual steering changes, including fallback, share a \SI{1}{rad/s} rate limit.

We adapt the local regression predictor of Rosolia and Borrelli's racing LMPC~\cite{rosolia2020} to the global-coordinate state and controls, keeping the same initialization, wide bounds, QP objective, fallback logic, and lap-acceptance rules as the fixed-wide model-based controller. While position and heading follow linearized kinematics, the dynamic-state increments are fitted via Epanechnikov-weighted affine regression over the 40 nearest samples from the two most recent accepted laps, with ridge regularization of $10^{-3}$ on the scaled slopes. Neighbor distances in $(v,r,\beta)$ are scaled by (\SI{1}{m/s}, \SI{2}{rad/s}, \SI{0.2}{rad}) with a kernel bandwidth equal to $\sqrt{1.25}$ times the farthest selected distance (squared-bandwidth floor $10^{-8}$); acceleration (scaled by \SI{2}{m/s^2}) enters the speed regressor, and steering (scaled by \SI{0.2}{rad}) enters the yaw-rate and sideslip regressors.

\subsubsection{Friction and Perturbation Tests}\label{sec:perturbations}

The path-following initialization is used for the sweep with matched plant and predictor friction at $\mu=0.60$, 0.80, 1.00, and 1.20. With the full-state initialization, the controller friction estimate is held at $\hat\mu=0.60$ as plant friction varies from 0.55 to 0.60 and 0.65. The latter comparison isolates the change in actual tire capacity while holding the nominal model fixed.

At $\mu=\hat\mu=0.60$, the tire-law test varies how lateral force builds up with slip. The plant retains the nominal mass, geometry, longitudinal dynamics, and combined-force limit. Its lateral force at each axle $i$ is
\begin{equation}
F_{yi}=-L_i\frac{z_i}{\sqrt{1+z_i^2}},\qquad
z_i=\frac{C_{\alpha i}\tan\alpha_i}{L_i},\qquad L_i=\bar\mu_i F_{zi}.
\label{eq:smooth_tire}
\end{equation}
with $F_{yi}=0$ whenever $L_i=0$. This law preserves the nominal small-slip cornering stiffness and approaches the same lateral-force limit (reaching $0.949L_i$ at the Fiala sliding threshold $|z_i|=3$), while the predictor and online force estimator retain the Fiala law.

A separate set of tests retains the Fiala plant while adding independent zero-mean Gaussian noise (standard deviations \SI{0.02}{rad/s} in yaw rate and $0.5^\circ$ in sideslip) to the controller's state feedback at every update across three random seeds, with pose and speed unperturbed. The noisy estimates feed MPC, fallback, acceptance checks, and terminal memory, whereas offline drift identification always evaluates the true plant states and tire forces. Finally, detector sensitivity is checked by varying $\beta_{\rm th}$, $\rho_{\rm th}$, and $T_{\rm d}$ one at a time.

\subsection{Results and Analysis}\label{sec:results}

\subsubsection{Onset of Drift and Rear-Axle Response}

Sustained drift develops as the joint interval grows across repeated passages through the same corners. The full handling envelope is available on Lap~8, when peak sideslip reaches $9.4^\circ$ and peak rear utilization is 0.758. The track sections above the sideslip threshold extend beyond those with high rear utilization (Fig.~\ref{fig:low_friction}(a,b)). Near $s=\SI{37.8}{m}$, the yaw condition is already met on Lap~8, and sideslip exceeds its threshold during part of the passage. Rear utilization reaches its threshold on Lap~10, when the three conditions coincide for up to \SI{0.25}{s}. The interval grows to \SI{0.40}{s} on Lap~11 and \SI{0.55}{s} on Lap~12 (Fig.~\ref{fig:low_friction}(c)). Both laps contain sustained episodes around $s=\SI{37}{m}$ and $s=\SI{47}{m}$ and finish with positive track margins and successful solves throughout.

\begin{figure}[!t]
  \centering
  \includegraphics[width=142mm]{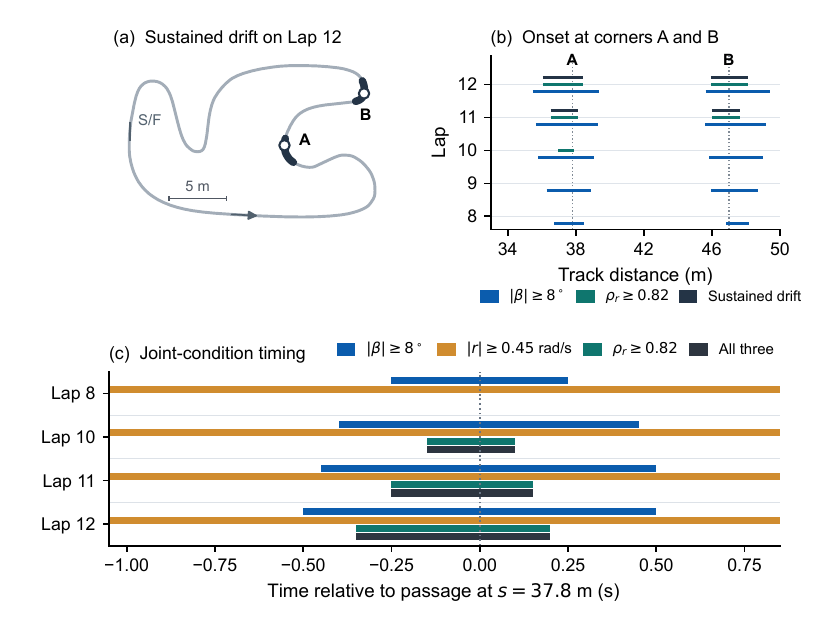}
  \caption{Spatial and temporal development of sustained drift at $\mu=0.60$. (a) Track plan with Lap~12 episodes projected onto the centerline. Thick segments mark sustained drift, and S/F denotes start/finish. (b) Track-distance intervals meeting the sideslip and rear-utilization thresholds near A and B, and intervals that qualify as sustained drift. (c) Threshold intervals at corner A aligned to passage at $s=\SI{37.8}{m}$. The dark strip marks simultaneous overlap. Sustained drift requires at least \SI{0.30}{s}.}
  \label{fig:low_friction}
\end{figure}

At $s=\SI{37.8}{m}$, speed rises from \SI{1.65}{m/s} on Lap~3 to \SI{2.83}{m/s} on Lap~12. Sideslip changes from $-0.94^\circ$ to $16.38^\circ$, and rear utilization rises from 0.21 to 0.89. The normalized incremental rear stiffness falls from 0.86 to 0.27 (Fig.~\ref{fig:state_control}). The faster passage raises lateral-force demand at the same corner. The rear tires operate on the flattening portion of the Fiala force curve, where larger changes in slip accompany smaller changes in force. The positive incremental stiffness at this station places the Lap~12 response on the rising branch near the force limit.

\begin{figure}[!t]
  \centering
  \includegraphics[width=120mm]{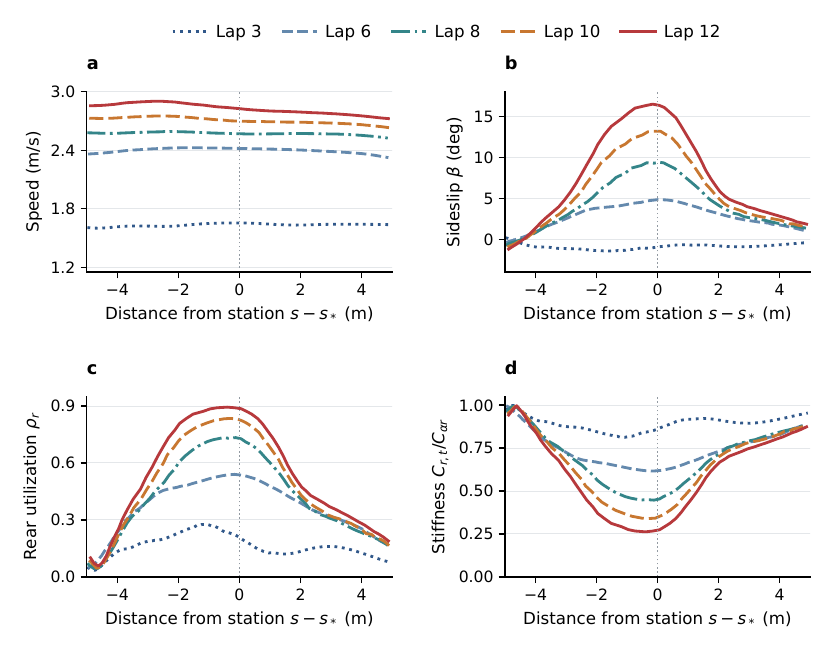}
  \caption{Tire and motion signatures at the same corner. (a) Speed. (b) Sideslip. (c) Rear-axle utilization. (d) Normalized incremental rear stiffness. Dotted lines mark $s=\SI{37.8}{m}$.}
  \label{fig:state_control}
\end{figure}

\subsubsection{Handling Bounds, Terminal Experience, and Prediction}

The handling bounds shape the achievable rear-tire demand. At $\mu=0.60$, the two variants with a $7^\circ$ predicted sideslip bound exceed $8^\circ$ in realized sideslip on some laps. Their peak rear utilizations on Lap~12 are 0.697 and 0.727, both below the drift threshold of 0.82. BE-LMPC reaches 0.894 and sustains the joint drift conditions (Fig.~\ref{fig:method}).

\begin{figure}[!t]
  \centering
  \includegraphics[width=142mm]{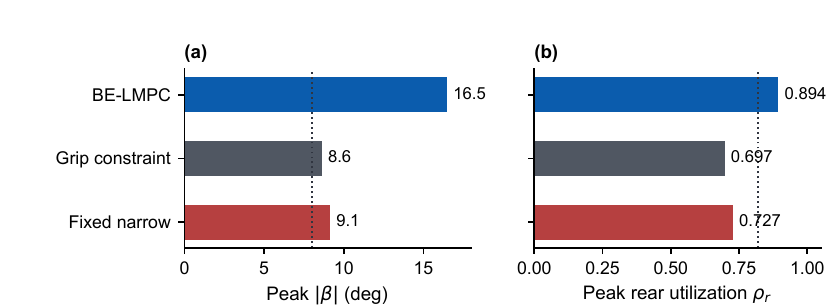}
  \caption{Effect of restricting the predicted handling bounds at $\mu=0.60$ on Lap~12. (a) Peak absolute vehicle sideslip. (b) Peak rear-axle utilization. The dotted lines mark the individual thresholds used in the joint drift criterion.}
  \label{fig:method}
\end{figure}

The fixed-wide variant admits large sideslip from Lap~3, yet its realized peak is only $4.5^\circ$ on that lap. Its longest joint interval grows from \SI{0.15}{s} on Lap~10 to \SI{0.40}{s} on Lap~11 and \SI{0.50}{s} on Lap~12, matching the onset lap of the adaptive envelope under the same full-state initialization and showing that drift still develops progressively even when the handling envelope is wide from the outset (Fig.~\ref{fig:envelope_experience}).

The adaptive and frozen-experience runs share identical trajectories through Lap~8 and continue under the same full envelope. By Lap~12, peak rear utilization reaches 0.795 with frozen experience versus 0.882 with terminal updates, with corresponding peak sideslips of $11.4^\circ$ and $15.8^\circ$; only the updating controller sustains the joint drift conditions on Laps~11 and~12. At the same corner on Lap~12, speed rises from \SI{2.62}{m/s} to \SI{2.79}{m/s} and sideslip from $11.40^\circ$ to $15.72^\circ$, confirming that newly accumulated terminal experience, rather than envelope relaxation alone, drives the vehicle into sustained drift.

\begin{figure}[!t]
  \centering
  \includegraphics[width=142mm]{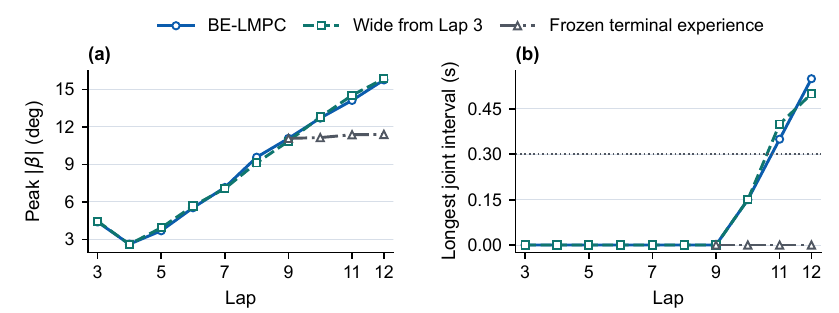}
  \caption{Tests of handling bounds and terminal experience at $\mu=0.60$ with shared full-state initialization. (a) Peak absolute sideslip by lap. (b) Longest joint drift-condition interval and the \SI{0.30}{s} criterion. The frozen-experience curve starts after the common trajectory through Lap~8.}
  \label{fig:envelope_experience}
\end{figure}

Under the same wide bounds at $\mu=0.60$, replacing the physics-based predictor with local affine regression substantially alters the learned response. Whereas the fixed-wide model-based controller reaches 0.885 peak rear utilization and a \SI{0.50}{s} joint interval on Lap~12, the adapted racing LMPC peaks at $4.5^\circ$ sideslip and 0.641 rear utilization, with zero qualifying drift duration throughout the ten-lap window (Table~\ref{tab:friction_mismatch}). Thus, accurate representation of nonlinear tire saturation in the prediction model directly governs whether iterative progress pushes rear-axle demand into the drift regime.

\begin{table}[!t]
\centering
\small
\setlength{\tabcolsep}{4pt}
\caption{Drift measures with fixed nominal friction estimate $\hat\mu=0.60$. The peaks and longest interval $D_{\max}$ are for Lap~12. The minimum signed margin, accepted laps, and qualifying laps cover Laps~3--12.}
\label{tab:friction_mismatch}
\begin{tabular}{@{}clrrrrrl@{}}
\toprule
Plant $\mu$ & Controller & \shortstack{Peak $|\beta|$\\(deg)} & Peak $\rho_r$ & \shortstack{$D_{\max}$\\(s)} & \shortstack{Min.\ margin\\(mm)} & Accepted & \shortstack{Qualifying\\laps}\\
\midrule
0.55 & BE-LMPC & 18.8 & 0.954 & 0.75 & +1.2 & 10/10 & 9, 10, 11, 12 \\
0.55 & Adapted racing LMPC & 4.7 & 0.683 & 0.00 & -22.5 & 8/10 & -- \\
0.60 & BE-LMPC & 15.8 & 0.882 & 0.55 & +0.6 & 10/10 & 11, 12 \\
0.60 & Adapted racing LMPC & 4.5 & 0.641 & 0.00 & -21.2 & 9/10 & -- \\
0.60 & Fixed-wide LMPC & 15.9 & 0.885 & 0.50 & -1.6 & 10/10 & 11, 12 \\
0.65 & BE-LMPC & 13.5 & 0.801 & 0.00 & -3.8 & 10/10 & -- \\
0.65 & Adapted racing LMPC & 4.0 & 0.587 & 0.00 & -20.5 & 9/10 & -- \\
\bottomrule
\end{tabular}
\end{table}

\subsubsection{Spatial Spread During Continued Learning}

Drift extends to more corners and lasts longer within individual corners. The nominal run contains two sustained episodes totaling \SI{0.70}{s} on Lap~11. By Lap~16, five episodes at distinct locations total \SI{3.90}{s}. On Lap~30, the same five locations remain active, total duration reaches \SI{6.00}{s}, and the longest episode lasts \SI{1.50}{s} (Fig.~\ref{fig:continuation}(a,c)).

At $s=\SI{37.8}{m}$, the ratio of rear lateral force magnitude to available lateral capacity rises from 0.886 on Lap~12 to 0.999 on Lap~30. The rear slip index grows from 0.51 to 0.92, incremental rear stiffness falls from 0.274 to 0.011 in normalized units, and sideslip grows from $16.3^\circ$ to $31.1^\circ$. The operating points approach the force plateau along the Fiala curve (Fig.~\ref{fig:continuation}(b)). Countersteer first occurs during sustained drift on Lap~24, accompanying the later high-sideslip response~\cite{voser2010highsideslip}.

\begin{figure}[!t]
  \centering
  \includegraphics[width=142mm]{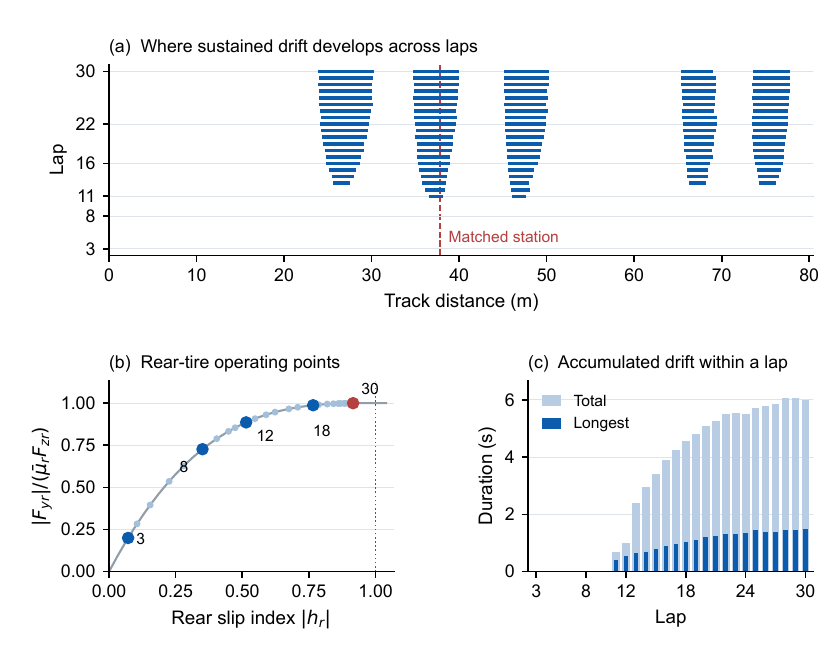}
  \caption{Drift progression through Lap~30 with the nominal controller at $\mu=0.60$. (a) Observed track-distance spans of sustained episodes. The dashed vertical line marks $s=\SI{37.8}{m}$. (b) Operating points at that station from each lap, overlaid on the normalized Fiala law. Numbers denote representative laps. (c) Total and longest episode duration on each lap.}
  \label{fig:continuation}
\end{figure}

\subsubsection{Tire-Force Capacity and Onset Sensitivity}\label{sec:perturbation_results}

At fixed $\hat\mu=0.60$, sustained drift first occurs on Lap~9 at plant friction 0.55 and on Lap~11 at 0.60. At 0.65, the longest joint interval stays below the duration criterion in the observation window (Fig.~\ref{fig:friction}(a)). On Lap~12, speeds at the matched station lie between 2.790 and \SI{2.798}{m/s}, while absolute sideslip falls from $18.8^\circ$ to $15.7^\circ$ and $13.4^\circ$ as friction increases. Rear utilization falls from 0.953 to 0.873 and 0.786, crossing the rear-utilization threshold (Fig.~\ref{fig:friction}(b)). The wider matched-friction sweep exhibits the same trend. At approximately \SI{2.97}{m/s} peak speed, raising friction from 0.60 to 0.80, 1.00, and 1.20 reduces Lap~12 peak rear utilization from 0.894 to 0.699, 0.572, and 0.483, respectively, while peak sideslip exceeds $8^\circ$ in all four cases (Fig.~\ref{fig:friction}(c)). This surface dependence agrees with the minimum-time cornering solutions of Tavernini et al.~\cite{tavernini2013minimum}. The matched-corner comparison here links the onset shift to the fraction of available rear force used.

\begin{figure}[!t]
  \centering
  \includegraphics[width=142mm]{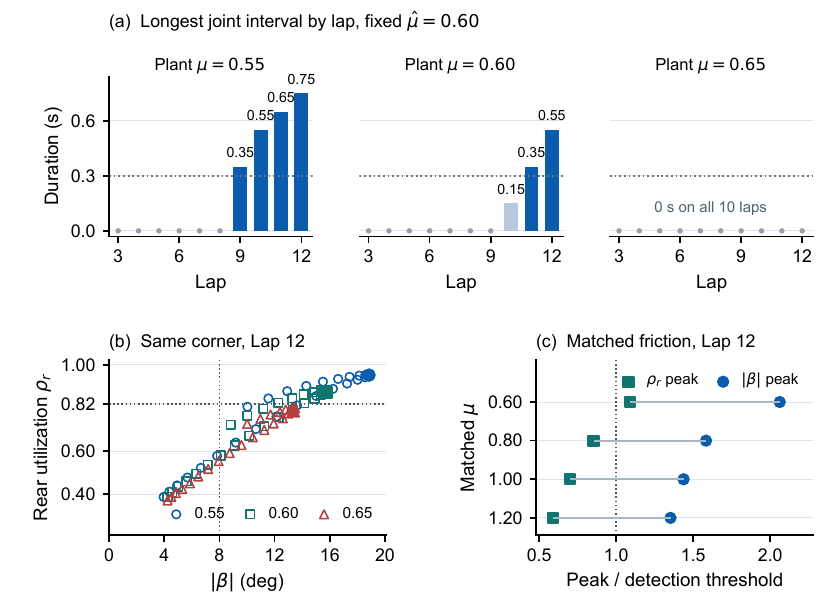}
  \caption{Role of tire-force capacity in drift emergence. (a) Longest simultaneous-threshold interval in seconds for each lap with fixed $\hat\mu=0.60$. The dotted line marks the \SI{0.30}{s} duration criterion, and filled blue bars meet it. (b) Lap~12 samples over $s=35.5$--\SI{40.5}{m} at the same three plant frictions. Filled markers denote the samples nearest $s=\SI{37.8}{m}$. (c) Lap~12 peaks normalized by their respective offline thresholds in the matched-friction sweep. The two peaks can occur at different times.}
  \label{fig:friction}
\end{figure}

At fixed friction, switching the plant to the smooth tire law delays drift onset by one lap, from Lap~11 to Lap~12: peak sideslip already reaches $15.2^\circ$ on Lap~11 with a \SI{0.25}{s} joint interval, and on Lap~12 peak rear utilization rises from 0.830 to 0.855 with a \SI{0.50}{s} interval. Across the three noisy-state runs, sustained drift still begins on Lap~11 and reaches \SI{0.50}{s}--\SI{0.55}{s} on Lap~12. All four perturbed runs complete ten accepted learning laps, with a worst-case signed track margin of \SI{-3.5}{mm} (well within the \SI{20}{mm} tolerance; Table~\ref{tab:perturbations}).

\begin{table}[H]
\centering
\small
\setlength{\tabcolsep}{5pt}
\caption{Drift emergence with plant or observation perturbations at $\mu=\hat\mu=0.60$. Peak $\rho_r$ and $D_{\max}$ are from Lap~12. Minimum margin and accepted laps cover Laps~3--12.}
\label{tab:perturbations}
\begin{tabular}{@{}lcrrrc@{}}
\toprule
Case & \shortstack{First qualifying\\lap} & \shortstack{$D_{\max}$\\(s)} & Peak $\rho_r$ & \shortstack{Min.\ margin\\(mm)} & Accepted\\
\midrule
Matched Fiala & 11 & 0.55 & 0.882 & +0.6 & 10/10 \\
Smooth plant tire law & 12 & 0.50 & 0.855 & -2.7 & 10/10 \\
State noise (seed 2026) & 11 & 0.55 & 0.883 & +0.9 & 10/10 \\
State noise (seed 2027) & 11 & 0.50 & 0.887 & -3.5 & 10/10 \\
State noise (seed 2028) & 11 & 0.50 & 0.890 & +0.4 & 10/10 \\
\bottomrule
\end{tabular}
\end{table}

The reported onset is unchanged across sideslip thresholds from $6^\circ$ to $12^\circ$ in the nominal sensitivity test. A rear-utilization threshold of 0.75 yields Lap~9, the nominal threshold of 0.82 yields Lap~11, and 0.90 exceeds the observed peak. Increasing the required duration from \SI{0.20}{s} to \SI{0.50}{s} moves onset from Lap~10 to Lap~12 (Table~\ref{tab:detector_sensitivity}). Varying these criteria therefore only shifts the cut point along the continuous growth of the joint interval in Fig.~\ref{fig:low_friction}.

\begin{table}[!t]
\centering
\small
\setlength{\tabcolsep}{6pt}
\caption{Sensitivity of the first qualifying lap to the offline drift criteria}
\label{tab:detector_sensitivity}
\begin{tabular}{lll}
\toprule
Varied parameter & Values & First qualifying lap\\
\midrule
$\beta_{\rm th}$ / ($^\circ$) & 6, 8, 10, 12 & 11, 11, 11, 11\\
$\rho_{\rm th}$ & 0.75, 0.82, 0.90 & 9, 11, --\\
$T_{\rm d}$ / s & 0.20, 0.30, 0.50 & 10, 11, 12\\
\bottomrule
\end{tabular}
\end{table}

\FloatBarrier

\begin{samepage}
Figure~\ref{fig:corner_poses} compares the trajectories and Lap~12 vehicle poses at two corners with $\hat\mu=0.60$. At corner A with $\mu=0.60$, the two controllers follow similar Lap~12 paths, with sideslip magnitudes of $15.7^\circ$ for BE-LMPC and $4.5^\circ$ for the adapted racing LMPC at the marked station.

\begin{figure}[H]
  \centering
  \includegraphics[width=\linewidth]{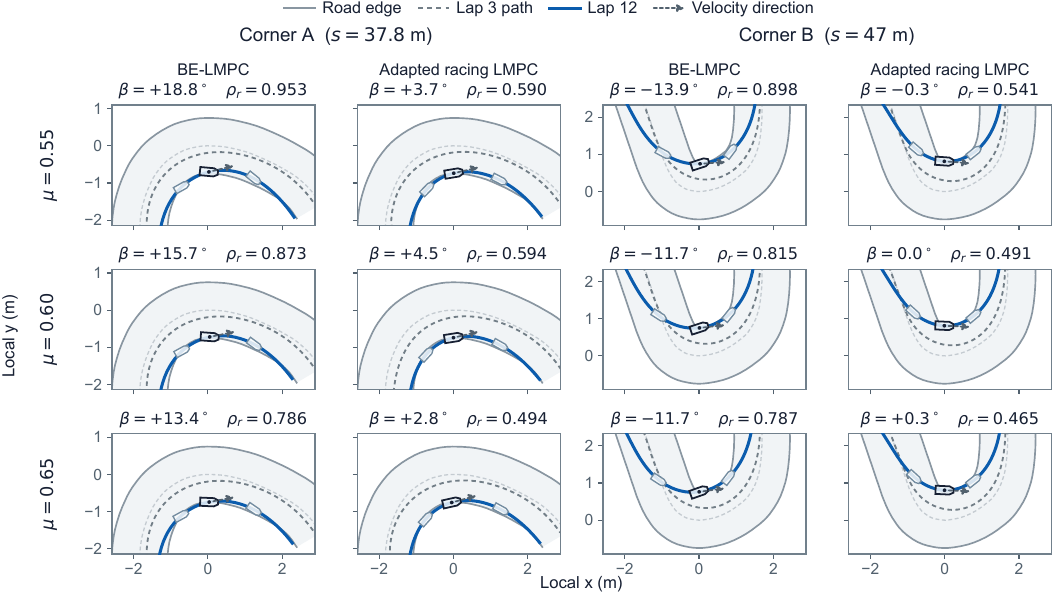}
  \caption{Lap~3 and Lap~12 trajectories at corners A and B for three plant friction coefficients. The vehicle outlines and dashed arrows show the Lap~12 body pose and velocity direction, respectively. Sideslip $\beta$ and rear utilization $\rho_r$ are evaluated at the sample nearest each marked station.}
  \label{fig:corner_poses}
\end{figure}
\end{samepage}

\section{Hardware-in-the-Loop Validation}\label{sec:hil}

\subsection{Speedgoat and Simulink Real-Time Setup}\label{sec:hil_setup}

The HIL bench runs the single-track/Fiala plant on a Speedgoat Baseline target via Simulink Real-Time (MATLAB R2025a) using a fixed \SI{5}{ms} RK4 step, while a Windows host (Intel Core Ultra~7~155H) executes BE-LMPC in Python~3.11 with OSQP. Host and target exchange states and commands over UDP at a \SI{50}{ms} control period (Fig.~\ref{fig:hil_setup}); commands are applied one period after state sampling, and the target holds the previous input if a deadline is missed.

For direct comparison with a matched desktop simulation, both runs use $\mu=\hat\mu=0.60$ and share the plant model, one-period command delay, grip-driving initialization, and progress objective in \eqref{eq:objective}, with drift identified offline as in Section~\ref{sec:formulation}.

\begin{figure}[H]
  \centering
  \includegraphics[width=142mm]{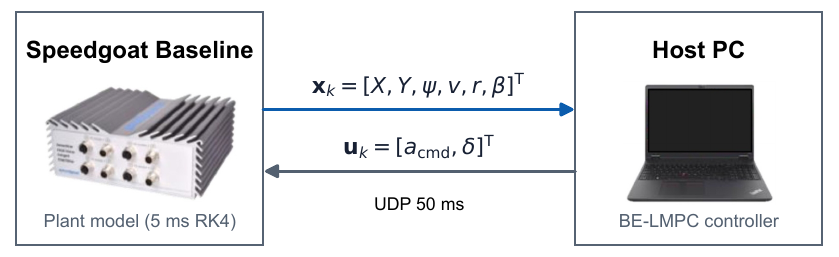}
  \caption{Speedgoat HIL setup with UDP exchange of vehicle states and control commands.}
  \label{fig:hil_setup}
\end{figure}

\subsection{Drift Emergence Under Nominal Conditions}\label{sec:hil_nominal}

Both runs reach the full handling envelope on Lap~8 and first sustain drift on Lap~10. In the HIL run, the longest joint interval grows from \SI{0.05}{s} on Lap~9 to \SI{0.35}{s} on Lap~10 and \SI{0.60}{s} on Lap~12, closely matching the desktop values of \SI{0.30}{s} on Lap~10 and \SI{0.60}{s} on Lap~12 (Fig.~\ref{fig:hil_drift}(a)).

At corner B ($s=\SI{47.0}{m}$), both runs already satisfy the sideslip and yaw-rate thresholds on Lap~9 while rear utilization sits at 0.813. On Lap~10, utilization rises to 0.843 in both runs (Fig.~\ref{fig:hil_drift}(b)), sustaining the joint drift conditions for \SI{0.35}{s} in HIL and \SI{0.30}{s} on the desktop with no missed command deadlines during the episode.

\begin{figure}[!ht]
  \centering
  \includegraphics[width=142mm]{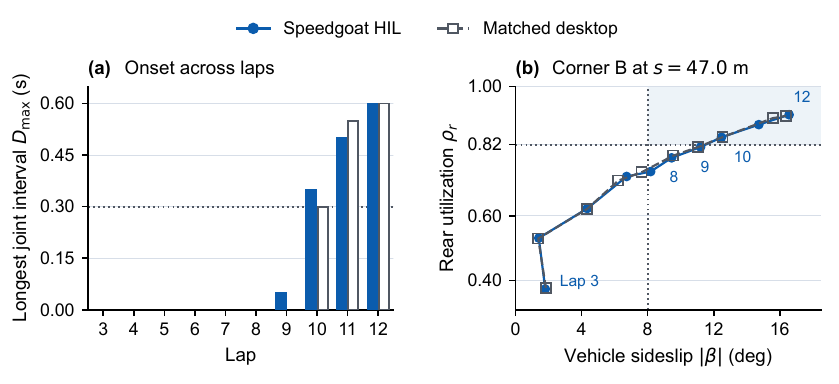}
  \caption{Drift emergence in the Speedgoat trial and the matched desktop run at $\mu=\hat\mu=0.60$. (a) Longest joint interval on Laps~3--12. The dotted line marks the \SI{0.30}{s} duration criterion. (b) Vehicle sideslip and rear utilization interpolated at $s=\SI{47.0}{m}$ in corner B. The shaded area meets both thresholds, and yaw rate meets its threshold at this station on all displayed laps. Connecting lines follow lap order, and labels identify selected HIL laps.}
  \label{fig:hil_drift}
\end{figure}

\begin{samepage}
\subsection{Control-Cycle Timing and Lap-Boundary Updates}\label{sec:hil_timing}

Table~\ref{tab:hil_timing} summarizes real-time execution across Laps~3--12. The Speedgoat base task completes within \SI{0.195}{ms}, while host execution (QP assembly, OSQP solve, and state/command handling) has a median cycle time of \SI{30.86}{ms} and a 95th percentile of \SI{41.25}{ms}. Across 6,315 control requests, the target holds the previous command on only 62 cycles (0.98\%), all ten learning laps are accepted, and all nine on-time lap-boundary transitions return within \SI{35}{ms}, well inside the \SI{50}{ms} control period.

\begin{table}[H]
\centering
\small
\caption{HIL timing on Laps~3--12. Command return is timed from state transmission to command reception for the 6,253 on-time requests, with \SI{5}{ms} resolution.}
\label{tab:hil_timing}
\begin{tabular}{@{}lrrr@{}}
\toprule
Operation & Median (ms) & 95th percentile (ms) & Maximum (ms)\\
\midrule
Speedgoat base task & 0.034 & 0.083 & 0.195\\
Host MPC solve & 30.20 & 40.58 & 69.28\\
Host controller & 30.86 & 41.25 & 69.99\\
Command return & 35 & 45 & 50\\
Host lap-boundary update & 0.252 & 0.317 & 0.325\\
\bottomrule
\end{tabular}
\end{table}
\end{samepage}

\FloatBarrier

\section{Conclusion}\label{sec:conclusion}

Addressing how sustained drifting arises organically from grip driving without a sideslip reference or drift reward, this study demonstrates that drift emergence is a continuous, conditional closed-loop dynamical phenomenon driven by iterative task demand and rear-axle tire saturation rather than an isolated control mode. Within vehicle dynamics and limit-handling control, these findings unite grip cornering and high-sideslip drifting along a single combined-slip continuum, showing that boundary-exploring predictive control can harness the friction limit from accumulated experience without mode switching or hand-crafted drift targets. This work corroborates prior observations of task-induced drift on low-adhesion surfaces~\cite{tavernini2013minimum,weiss2026emergent} and extends drift-tracking literature~\cite{goh2024,weber2024,stano2024,wu2026driftsurvey} by revealing the cross-lap onset mechanism, while showing that unlike local regression-based racing LMPC~\cite{rosolia2020}, modeling nonlinear tire saturation is decisive for pushing iterative progress into the drift regime. Nevertheless, the current validation rests on a planar single-track model with lumped actuation, uniform friction, and HIL plant emulation, leaving full-scale vehicle tests, spatially varying adhesion, roll-suspension coupling, and real actuator nonlinearities outside the present scope. Future work will combine online friction estimation, multi-body vehicle dynamics, and torque-vectoring actuation on physical vehicles to examine drift emergence over heterogeneous surfaces and in coupled emergency collision-avoidance scenarios.

 \section*{Acknowledgement}
 This work was supported by the National Natural Science Foundation of China under Grants 52505117 and 52472413, and partly supported by Grants-in-Aid for Scientific Research No. 25K07790.

\section*{Declarations}
\noindent\textbf{Competing Interests:} The authors declare that they have no competing interests.

\noindent\textbf{Declaration of Generative AI Use:} During the preparation of this manuscript, the authors used GPT-6 to refine grammar, readability, and phrasing. After using this tool, the authors reviewed and edited the manuscript as needed and take full responsibility for the content of the publication.

\begingroup
\renewcommand{\baselinestretch}{1}
\setlength{\parskip}{0pt}
\bibliographystyle{unsrtnat}
\bibliography{sn-bibliography}
\endgroup

\end{document}